\documentclass{article}
\usepackage{spconf,amsmath,graphicx,hyperref,booktabs, amssymb, amsthm}

\newtheorem{theorem}{Theorem}[section]
\newtheorem{definition}{Definition}[section]

\newtheorem{lemma}[theorem]{Lemma}

\DeclareMathOperator*{\argmin}{arg\,min}

\DeclareMathOperator{\Var}{Var}
\DeclareMathOperator{\Corr}{Corr}

\title{Rho-Perfect: Correlation Ceiling For Subjective Evaluation Datasets}
\name{Fredrik Cumlin}
\address{School of Electrical Engineering and Computer Science, KTH Royal Institute of Technology, Sweden}
\begin{document}
%
\maketitle
\begin{abstract}
Subjective ratings contain inherent noise that limits the model-human correlation, but this reliability issue is rarely quantified. In this paper, we present $\rho$-Perfect, a practical estimation of the highest achievable correlation of a model on subjectively rated datasets. We define $\rho$-Perfect to be the correlation between a perfect predictor and human ratings, and derive an estimate of the value based on heteroscedastic noise scenarios, a common occurrence in subjectively rated datasets. We show that $\rho$-Perfect squared estimates test-retest correlation and use this to validate the estimate. We demonstrate the use of $\rho$-Perfect on a speech quality dataset and show how the measure can distinguish between model limitations and data quality issues.
\end{abstract}
\section{Introduction}

Developing objective emulators for subjective opinions is a rich research field with various applications. Examples include speech quality assessment \cite{sqa_review}, image aesthetics \cite{image_aesthetics_review}, and recommendation systems \cite{recommender_systems}. When developing these models, a fundamental question arises: what is the highest correlation any model can achieve with human ratings? Ratings contain inherent noise, which creates an upper bound on model-human correlation that is less than $1.0$. Quantifying this bound is both challenging and important for effective model development.

Despite the widespread use of subjective ratings in ML evaluation, the reliability of the ratings is often overlooked. Recent surveys on speech and image aesthetics \cite{sqa_review, image_aesthetics_review, mittag2022deep} extensively cover evaluation methodologies but omit reliability measures and correlation ceilings. Ignoring rating reliability might lead to misguided conclusions about model performance. For example, an objective model might perform well overall but poorly under certain conditions, but this could reflect poor rating reliability rather than model limitations.


Existing reliability measures such as Pearson's correlation ratio ($\eta^2$) \cite{Pearson1905}, the family of intraclass correlations (ICC) \cite{icc_measures}, and Cronbach's alpha \cite{cronbach1951coefficient} have limitations in this setting, since they assume homoscedastic noise (equal noise variance across items) \cite{lord1968statistical} or are difficult to interpret with respect to model performance. Subjectively rated datasets in, for example, speech quality assessment and recommendation systems typically have an uneven number of ratings per item, and different items can have larger or smaller disagreement among raters. Generalizability theory provides measures under heteroscedastic noise assumptions, but has a theoretical disposition and is not discussed with respect to model performance \cite{brennan2001, kim2007}, which might be the reason that they are not used in the context of model performance in these scenarios.

In this paper, we introduce $\rho$-Perfect, a practical upper bound of model-human correlation for subjectively rated datasets. The estimate can be calculated from a single evaluation knowing only the distribution of the individual ratings per item. Moreover, squaring $\rho$-Perfect approximates the correlation between two independent subjective evaluations, enabling empirical validation. In short, $\rho$-Perfect provides a principled way to interpret model performance relative to the limits of human ratings, enabling differentiation between model limitations and data quality issues. Our contributions are as follows: (1) $\rho$-Perfect, a practical upper bound on model-human correlation for unbalanced subjectively rated datasets; (2) a theoretical link to test-retest correlation and empirical validation thereof; and (3) case studies showing how $\rho$-Perfect estimates can separate model limitations from data quality.

\section{The $\rho$-Perfect metric}

Consider a subjectively rated dataset $\mathcal{D}$. It consists of items which we denote by $\mathcal{X}=\{x_{i}\}_{i=1}^n$ where $n$ is the number of items. It also consists of a set of $m\in\mathbb{N}$ human raters that provide subjective ratings on the items. Since human annotation is costly, the most common rating strategy is to have $m_{i}\ll m$ raters provide with ratings on item $x_i$. Thus, a subjectively rated dataset consists of pairs $(x_i, r_i^{(j)})$, where $x_i$ is an item and $r_i^{(j)}\in \mathbb{R}$ is a rating provided by rater $(j)$. We define the average rating to be given by $y_i=\frac{1}{m_i}\sum_{j=1}^{m_i}r_i^{(j)}$.

\subsection{Mathematical Derivation of $\rho$-Perfect}

We first formalize the notion of a perfect predictor, which serves as the basis for defining $\rho$-Perfect.

Let $X$ denote a random variable representing the items. Let $Y$ denote a random variable such that $Y\vert X=x_i$ represents the distribution of the average rating $y_i$.

We define the perfect predictor as \begin{equation}
    \hat{f} \triangleq \argmin_{f}\mathbb{E}[(f(X)-Y)^2].
\end{equation} It is a standard result that the minimizer is the regression function $\hat{f}(X) = \mathbb{E}[Y \vert X]$ \cite[Th. 1.4.1]{bickel2015mathematical}. For brevity, we denote $\hat{Y}\triangleq \hat{f}(X) = \mathbb{E}[Y \vert X]$.

We are interested in the Pearson Correlation Coefficient of the perfect predictor and the average ratings given. This can be viewed as a correlation ceiling. We have a Lemma.

\begin{lemma}
    Let $X,Y$ be two random variables and $\hat{Y}=\mathbb{E}[Y\vert X]$. Then the correlation of $\hat{Y}$ and $Y$ is given by \begin{equation}\label{eq:correlation}
    \begin{split}
        \textnormal{Corr}(Y, \hat{Y}) = \frac{\textnormal{Cov}(Y, \hat{Y})}{\sqrt{\textnormal{Var}(Y)\textnormal{Var}(\hat{Y})}} =\sqrt{\frac{\textnormal{Var}(\hat{Y})}{\textnormal{Var}(Y)}}.
    \end{split}
    \end{equation}
\end{lemma}
\begin{proof}
If we show that $\textnormal{Cov}(Y, \hat{Y})=\text{Var}(\hat{Y})$, the second equality in Eq.~\ref{eq:correlation} follows and we are done.

From the law of total expectation, we have $$\mathbb{E}[Y]=\mathbb{E}\left[\mathbb{E}[Y\vert X]\right]=\mathbb{E}[\hat{Y}].$$ Further, we have $$\mathbb{E}[\hat{Y}Y]=\mathbb{E}[\mathbb{E}[Y\hat{Y}\vert X]]=\mathbb{E}[\hat{Y}\mathbb{E}[Y\vert X]]=\mathbb{E}[\hat{Y}^2].$$ It now follows that $$\textnormal{Cov}(Y, \hat{Y})= \mathbb{E}[\hat{Y}Y]-\mathbb{E}[\hat{Y}]\mathbb{E}[Y]=\mathbb{E}[\hat{Y}^2] - \mathbb{E}[\hat{Y}]^2 = \text{Var}(\hat{Y}).$$
\end{proof}

Estimating $\textnormal{Var}(Y)$ can be done using the unbiased variance estimate given the data, and is given by \begin{equation}
    \textnormal{Var}(Y) = \frac{1}{n-1}\sum_{i=1}^n (y_i-\overline{y})^2
\end{equation}
where $\overline{y}=1/n \sum_i y_i$. Estimating $\textnormal{Var}(\hat{Y})$ is more difficult as neither the distribution nor the perfect predictor is accessible in closed form. We estimate it in the following way. From the law of total variance, conditioning on $X$, we have 
\begin{equation}
\begin{split}
\textnormal{Var}(Y) &= \mathbb{E}[\textnormal{Var}(Y\vert X)] + \textnormal{Var}(\mathbb{E}[Y\vert X]) \\
&= \mathbb{E}[\textnormal{Var}(Y\vert X)] + \textnormal{Var}(\hat{Y}),
\end{split}
\end{equation}
hence, $\textnormal{Var}(\hat{Y}) = \textnormal{Var}(Y) - \mathbb{E}[\textnormal{Var}(Y\vert X)]$. Now, if we assume that the raters are uncorrelated given an item, the variance of the average rating $Y\vert X=x_i$ can be estimated using the 
standard result for sample means; if individual ratings for item $i$ 
have sample variance $s_{\textnormal{rating}}^2=\frac{1}{m_i-1}\sum(r^{(j)}_i - y_i)^2$, then:
\begin{equation}
    \textnormal{Var}(Y\vert X=x_i) =\frac{s_{\textnormal{rating}}^2}{m_i} = \frac{1}{m_i(m_i-1)}\sum_{j=1}^{m_i} (r_i^{(j)}-y_i)^2.
\end{equation}

Therefore, 
\begin{equation}
    \mathbb{E}\left[\Var(Y\vert X)\right] = \frac{1}{n}\sum_{i=1}^n \frac{1}{m_i(m_i-1)}\sum_{j=1}^{m_i} (r_i^{(j)}-y_i)^2.
\end{equation}

We conclude with the definition of $\rho$-Perfect.

\begin{definition}[$\rho$\textnormal{-Perfect}]\label{def:rho_perfect}
    Given a subjectively rated dataset $\mathcal{D}=\left\{\left(x_i, r_i^{(j)}\right)\right\}$, the $\rho$-Perfect metric is given by 
    \begin{equation}\label{eq:def_rho_prefect}
            \rho\textnormal{-Perfect} \triangleq \sqrt{\frac{\Var(\hat{Y})}{\Var(Y)}}
    \end{equation}
    where $\text{Var}(\hat{Y})$ is the variance of a perfect predictor, and $\text{Var}(Y)$ is the variance of the average ratings per item. They are estimated by the following formulae:
    \begin{equation}
\begin{split}
            \Var(Y) &= \frac{1}{n-1}\sum_{i=1}^n(y_i - \overline{y})^2 \\
            \Var(\hat{Y})
            &= \Var(Y) - \frac{1}{n}\sum_{i=1}^n \frac{1}{m_i(m_i-1)}\sum_{j=1}^{m_i}(r_i^{(j)}-y_i)^2
\end{split}
\end{equation}
\end{definition}

Formally, the definition coincides with Pearson's correlation ratio ($\eta^2$) but is extended to a heteroscedastic noise scenario \cite{lord1968statistical}.

From the definition of $\rho$-Perfect, we suggest that there should be at least $50$ items (as it is correlation-based), and that each item has at least $3$ ratings (as the variance of the average is estimated). The computational cost to estimate $\rho$-Perfect is $\mathcal{O}(M)$, where $M=\sum_i m_i$ is the total number of ratings.\footnote{Implementation of $\rho$-Perfect can be found at \url{https://github.com/fcumlin/rho-perfect}.}

\section{Experimental validation}

Direct validation of $\rho$-Perfect is difficult since the distribution $Y\vert X$ is unknown. However, we can do indirect verification by the following. 

Assume we have done two subjective evaluations on the same items. We model the evaluations by $Y_1$ and $Y_2$. Now, we assume that both evaluations yield the same perfect predictor, which is reasonable if the subjective evaluations are done using a similar group of raters with similar biases. This means $\hat{Y}=\mathbb{E}[Y_1\vert X]=\mathbb{E}[Y_2\vert X].$ Now, we have
\begin{equation}
\begin{split} \textnormal{Cov}(Y_1, Y_2) &= \mathbb{E}[Y_1Y_2] - \mathbb{E}[Y_1]\mathbb{E}[Y_2] \\
&= \mathbb{E}[\mathbb{E}[Y_1Y_2\vert X]] -\mathbb{E}[\hat{Y}]^2 \\
&= \mathbb{E}[\textnormal{Cov}(Y_1, Y_2\vert X)+\mathbb{E}[Y_1\vert X]\mathbb{E}[Y_2\vert X]] - \mathbb{E}[\hat{Y}]^2 \\
&= \underbrace{\mathbb{E}[\textnormal{Cov}(Y_1, Y_2\vert X)]}_{\approx 0} + \mathbb{E}[\hat{Y}^2]-\mathbb{E}[\hat{Y}]^2 \\
&= \textnormal{Var}(\hat{Y})
\end{split}
\end{equation}

We justify $\mathbb{E}[\textnormal{Cov}(Y_1, Y_2\vert X)]\approx 0$ empirically in Section 3.1 and provide intuition here. Given item $X$, the ratings thereof are the inherent noise of subjectivity, which could be perceptual differences (different views of the rating scale), personal differences (individual taste and standards), or simply noise in the human judgment (variability by external factors, such as time of day, when last meal occurred, etc.). Assuming the ratings are provided independently, the noises thereof are also expected to be independent. Note that $\mathbb{E}[\textnormal{Cov}(Y_1, Y_2\vert X)]\approx 0$ is only needed to validate $\rho$-Perfect, not to calculate it.

If we assume that the two subjective evaluations are similar, we have $\textnormal{Var}(Y_1)\approx\textnormal{Var}(Y_2)$. It follows that 
\begin{equation}
    \begin{split}
        \textnormal{Corr}(Y_1, Y_2) = \frac{\textnormal{Var}(\hat{Y})}{\sqrt{\textnormal{Var}(Y_1)\textnormal{Var}(Y_2)}} 
        \approx \frac{\textnormal{Var}(\hat{Y})}{\textnormal{Var}(Y_1)} 
        = \rho\textnormal{-Perfect}^2
    \end{split}
\end{equation}

where $\rho\textnormal{-Perfect}^2$ is the squared value of the $\rho\textnormal{-Perfect}$ measure using only the first subjective evaluation. This means that $\rho\textnormal{-Perfect}^2$ is an estimate of the correlation between two similar subjective evaluations on the same items. Note that $\rho\textnormal{-Perfect}^2$ is measured using \textbf{only} the first evaluation.

We conclude that if $\mathbb{E}[\textnormal{Cov}(Y_1, Y_2\vert X)]\approx 0$, $\rho\textnormal{-Perfect}$ is an appropriate measure of the expected PCC of a perfect predictor on the data, if $\rho\textnormal{-Perfect}^2$ is equal to the PCC of two similar subjective evaluations on the same data.

\subsection{Validating $\rho$-Perfect on real datasets}

We now validate $\rho$-Perfect on real datasets. To the author's knowledge, there are no subjectively rated datasets publicly available that have been rated twice. Thus, we simulate two subjective assessments from one assessment. We propose two methods. The first method is to extract all the raters and split them into two equally sized sets. The scores from the raters in the first set constitute the first assessment, and the scores from the raters in the second set constitute the second assessment. We call this simulation \textbf{Split-Raters}. The other method is to split the ratings per item into two equally sized sets. We call this simulation \textbf{Split-Ratings}. Note that for the second method, some raters might have been provided with ratings in both assessments. However, the second method ensures that we have an equal number of ratings per item.

The datasets used are BVCC, MovieLens, SOMOS, and MERP.

\textbf{BVCC:} The BVCC dataset  consists of $4,974$ speech clips with exactly $8$ speech quality ratings per clip \cite{bvcc}. The ratings are given on a discrete scale from $1$ to $5$. 

\textbf{MovieLens:} The MovieLens dataset is a movie dataset that human raters have provided with recommendation ratings \cite{movielens}. It consists of $1,349$ movies with an average of $74$ ratings per movie; the least rated movie has $5$ ratings, and the most rated movie has $583$ ratings. The ratings are given on a discrete scale from $1$ to $5$.

\textbf{SOMOS:} The SOMOS dataset is a speech dataset that human raters have annotated based on speech quality \cite{SOMOS}. It consists of $20,100$ speech clips with an average of $18$ ratings per speech clip; the least rated clip has $17$ ratings, and the most rated clip has $146$ ratings. The ratings are given on a discrete scale from $1$ to $5$.

\textbf{MERP:} The MERP dataset is a music dataset that humans have annotated based on emotional experience \cite{merp}. In this experiment, we only extract the 'arousal' ratings, which are ratings based on the intensity of emotion, a continuous value from $-1$ to $1$. For a song, a rater provides one rating per second. We consider the average of these ratings to be the arousal rating of the song. There are $60$ songs in total, with an average of $57$ ratings per song; the least rated song has $6$ ratings, and the most rated song has $100$ ratings.

We do $10$ iterations with different seeds when splitting the raters/ratings, and report the mean and standard deviation. The results can be seen in Table~\ref{tab:experimental_validation}.

\begin{table}[htbp]
\centering
\small
\setlength{\tabcolsep}{4pt}
\begin{tabular}{lccc}
\toprule
Dataset & $\mathbb{E}[\textnormal{Cov}(Y_1, Y_2\mid X)]$ & $\rho\textnormal{-Perfect}^2$ & Target \\
 & & & Corr$(Y_1, Y_2)$ \\
\midrule
\multicolumn{4}{c}{Split-Raters} \\
\midrule
BVCC & $0.0^*$ & $0.798_{\scriptstyle\pm 0.001}$ & $0.801_{\scriptstyle\pm 0.001}$ \\
MovieLens & $0.0^*$ & $0.734_{\scriptstyle\pm 0.001}$ & $0.728_{\scriptstyle\pm 0.001}$ \\
SOMOS & $0.0^*$ & $0.258_{\scriptstyle\pm 0.002}$ & $0.297_{\scriptstyle\pm 0.001}$  \\
MERP & $0.0^*$ & $0.499_{\scriptstyle\pm 0.020}$ & $0.502_{\scriptstyle\pm 0.008}$ \\
\midrule
\multicolumn{4}{c}{Split-Ratings} \\
\midrule
BVCC & $0.0^*$ & $0.800_{\scriptstyle\pm 0.001}$ & $0.800_{\scriptstyle\pm 0.001}$ \\
MovieLens & $0.0^*$ & $0.710_{\scriptstyle\pm 0.001}$ & $0.701_{\scriptstyle\pm 0.001}$ \\
SOMOS & $0.0^*$ & $0.281_{\scriptstyle\pm 0.001}$ & $0.281_{\scriptstyle\pm 0.001}$  \\
MERP & $0.0^*$ & $0.478_{\scriptstyle\pm 0.009}$ & $0.502_{\scriptstyle\pm 0.007}$ \\
\bottomrule
\end{tabular}
\caption{Comparison of the squared $\rho$-Perfect measure of a subjectively rated dataset and the correlation of the subjectively rated dataset to a similar one. Validation of the $\rho$-Perfect metric is done by comparison of the squared value to the correlation. $^*$ All values below $10^{-18}$; effectively zero within numerical precision. }
\label{tab:experimental_validation}
\end{table}

\begin{table*}[t]
\centering
\caption{Comparison of $\rho$-Perfect with existing reliability metrics on subjective datasets}
\label{tab:comparison}
\begin{tabular}{lccccc}
\toprule
Dataset & $\Corr(Y_1, Y_2)$ & ICC(2, k) & Subsampling reliability & $\rho$-Perfect$^2$ \\
\midrule
BVCC & $0.801_{\scriptstyle\pm 0.001}$ & $0.822_{\scriptstyle\pm 0.001}$ & $0.893_{\scriptstyle\pm 0.001}$ & $0.796_{\scriptstyle\pm 0.001}$ \\
MovieLens & $0.728_{\scriptstyle\pm 0.001}$ & $0.898_{\scriptstyle\pm 0.001}$ & $0.879_{\scriptstyle\pm 0.001}$ & $0.719_{\scriptstyle\pm 0.001}$  \\
SOMOS & $0.297_{\scriptstyle\pm 0.002}$ & $0.326_{\scriptstyle\pm 0.001}$ & $0.716_{\scriptstyle\pm 0.001}$ & $0.269_{\scriptstyle\pm 0.001}$ \\
MERP & $0.502_{\scriptstyle\pm 0.010}$ & $0.554_{\scriptstyle\pm 0.001}$ & $0.807_{\scriptstyle\pm 0.001}$ & $0.483_{\scriptstyle\pm 0.011}$ \\
\bottomrule
\end{tabular}
\end{table*}

As can be seen in Table~\ref{tab:experimental_validation}, the term $\mathbb{E}[\textnormal{Cov}(Y_1, Y_2\mid X)]$ is effectively zero, hence justified by both empirical investigation and intuitive reasoning. This means that $\rho$-Perfect squared should be a good estimator of the correlation between two subjective ratings. In both split-raters and split-ratings methods, we find that $\rho$-Perfect squared is a suitable estimator of this correlation. This validates the use of $\rho$-Perfect as an estimator of the correlation between a subjectively rated dataset and a perfect predictor thereof; informally, the highest achievable correlation on that dataset.

\subsection{Comparison to existing measures}

We now compare with existing reliability measures. We report the test-retest experimental correlation of the split-raters methodology for comparison. Note that this means we compute the reliability of 'half' the datasets; only the first evaluation is used for reliability computation.

\textbf{ICC(2, k):} Following the guideline in Table 3 and Figure 1 in \cite{icc_guideline}, we implement ICC(2, k), which measures test-retest reliability of the average ratings ($k$ ratings per item) in a two-way random effects model with absolute agreement. The measure assumes that for an item, only a subset of the raters rate the item, and that there is an equal number of raters per item. The last assumption is violated for most of the datasets we test on, which $\rho$-Perfect is designed to address. 

\textbf{Subsampling reliability:} The subsampling reliability measure was used in \cite{MOSNet} and \cite{SOMOS} for measuring the reliability of ratings in the speech quality assessment task. It is done by randomly drawing half of the raters from the full rater pool and computing the correlation between the average ratings of the two sets. This is done over several iterations, and the average correlation is reported as the reliability.

As can be seen in Table~\ref{tab:comparison}, ICC(2, k) and $\rho$-Perfect$^2$ generally agree, except for MovieLens, where ICC(2, k) counterintuitively suggests higher reliability than BVCC despite lower rater test-retest correlation. A likely culprit is that MovieLens has a large, varying number of raters per clip, hence violating the assumptions of the ICC(2, k) measure. $\rho$-Perfect does not suffer from this limitation.

The subsampling reliability measure seems to consistently overestimate the reliability, but follows the overall trend of reliability when studying the test-retest correlation. A likely reason for the overestimation is that the ratings from the subset are already present in all the ratings, resulting in dependent realisations which does not reflect a truly independent evaluation. 

\section{Interpreting model performance with $\rho$-Perfect}

\begin{table}[!t]
\centering
\caption{Model performance relative to $\rho$-Perfect upper bounds on NISQA validation data and different subsets.}
\label{tab:model_performance}
\begin{tabular}{lcc}
\hline
Condition & Model PCC & $\rho$-Perfect \\
\hline
All data & $0.873$ & $0.954$ \\
Bandpass filtered & $0.934$ & $0.969$ \\
Clean conditions & $0.621$ & $0.816$  \\
Bursty distortions & $0.392$ & $0.701$ \\
\hline
\end{tabular}
\end{table}

To demonstrate practical utility, we evaluate the objective speech quality model DNSMOS Pro \cite{DNSMOSp} on the NISQA validation dataset \cite{NISQA}. DNSMOS Pro predicts the overall speech quality given a speech clip, and the NISQA dataset consists of $2500$ speech clips, each of which is rated by $5$ raters according to the overall quality. Extracting subsets of the dataset and computing model performance can give insights into problematic speech distortion scenarios for DNSMOS Pro.


The results are shown in Table~\ref{tab:model_performance}. The PCC on all data is $0.873$ with a $\rho$-Perfect of $0.954$, which suggests strong model performance and reliable ratings. However, the performance varies significantly for the different subsets presented. For bandpass filtered clips, the PCC is high, close to $\rho$-Perfect. For clean clips and high-burst distorted clips, the model PCC is significantly poorer ($0.621$ and $0.392$, respectively), and $\rho$-Perfect suggests that the reliability of the data is also lower. For the clean condition, part of the performance degradation can be explained by a lower $\rho$-Perfect, but this could also constitute an improvement area for DNSMOS Pro. For the high-burst scenario, while the PCC is the lowest ($0.392$), the moderate $\rho$-Perfect ($0.701$) indicates that both model improvements and potentially more reliable subjective evaluation will benefit a more accurate diagnosis.

This analysis demonstrates a practical use case of $\rho$-Perfect to distinguish between dataset limitations and model deficiencies. When model correlation drops, $\rho$-Perfect can aid in determining whether this reflects model weaknesses or limitations in subjective evaluation reliability.

\section{Conclusion}
In this paper, we have derived $\rho$-Perfect, an estimation of the model-human correlation ceiling. We have shown that $\rho$-Perfect squared approximates the correlation between two independently, but similar, subjective evaluations. We have shown that this theoretical result holds on real datasets across different subjectively rated datasets and different reliabilities of the ratings thereof. We have compared $\rho$-Perfect to other reliability measures and shown that $\rho$-Perfect has a clear interpretation. Finally, we have shown how to use $\rho$-Perfect in practice by comparing the measure to model-human correlation on the NISQA dataset.

\vfill\pagebreak

\section{Acknowledgement}
This work was partially supported by the Wallenberg AI, Autonomous Systems and Software Program (WASP) funded by the Knut and Alice Wallenberg Foundation.

\bibliographystyle{IEEEbib}
\bibliography{strings,refs}

\end{document}